\documentclass[english, twosided, peerreview]{IEEEtran}
\usepackage{verbatim}
\usepackage[T1]{fontenc}
\usepackage[latin9]{inputenc}
\usepackage{color}
\usepackage{babel}
\usepackage{amsthm}
\usepackage{amsmath}
\usepackage{amssymb}
\usepackage{esint}

\makeatletter
  \theoremstyle{definition}
  \newtheorem{defn}{\protect\definitionname}
  \theoremstyle{remark}
  \newtheorem{rem}{\protect\remarkname}
\theoremstyle{plain}
\newtheorem{thm}{\protect\theoremname}
  \theoremstyle{plain}
  \newtheorem{cor}{\protect\corollaryname}
  \theoremstyle{plain}
  \newtheorem{lem}{\protect\lemmaname}

\usepackage{cite}
\usepackage[algo2e,algoruled,noline]{algorithm2e}
\usepackage{upgreek}
\usepackage{mymath}

\makeatother

  \providecommand{\definitionname}{Definition}
  \providecommand{\lemmaname}{Lemma}
  \providecommand{\remarkname}{Remark}
\providecommand{\corollaryname}{Corollary}
\providecommand{\theoremname}{Theorem}

\begin{document}

\title{Learning Model-Based Sparsity\\
via Projected Gradient Descent }

\author{Sohail~Bahmani%
\thanks{S.B. is with the School of Electrical and Computer Engineering at
Georgia Institute of Technology.%
}, Petros~T.~Boufounos, \IEEEmembership{Member,~IEEE}%
\thanks{P.B. is with Mitsubishi Electric Research Labs.%
}, and Bhiksha~Raj %
\thanks{B.R. is with the Language Technologies Institute and the Department
of Electrical and Computer Engineering at Carnegie Mellon University.%
}}
\maketitle
\begin{abstract}
Several convex formulation methods have been proposed previously for
statistical estimation with structured sparsity as the prior. These
methods often require a carefully tuned regularization parameter,
often a cumbersome or heuristic exercise. Furthermore, the estimate
that these methods produce might not belong to the desired sparsity
model, albeit accurately approximating the true parameter. Therefore,
greedy-type algorithms could often be more desirable in estimating
structured-sparse parameters. So far, these greedy methods have mostly
focused on linear statistical models. In this paper we study the projected
gradient descent with non-convex structured-sparse parameter model
as the constraint set. Should the cost function have a Stable Model-Restricted
Hessian the algorithm produces an approximation for the desired minimizer.
As an example we elaborate on application of the main results to estimation
in Generalized Linear Models.\end{abstract}


\section{\label{sec:Introduction}Introduction}

In a variety of applications such as bioinformatics, medical imaging,
social networks, and astronomy there is a growing demand for computational
methods that perform statistical inference on high-dimensional data.
In the problems arising in these applications, $p$, the number of
predictors in each sample is much larger than $n$, the number of
observations. Although such problems are generally ill-posed, in many
cases the data has known underlying structure such as sparsity that
can be exploited to make the problem well-posed.

Beyond the ordinary, extensively studied, sparsity model, a variety
of structured sparsity models have been proposed in the literature
\cite{bach_consistency_2008,roth_group-lasso_2008,jacob_group_2009,baraniuk_model-based_2010,bach_structured_2010,jenatton_structured_2011,kyrillidis_combinatorial_2012,chandrasekaran_convex_2010}.
These sparsity models are designed to capture the interdependence
of the locations of the non-zero components that is known \emph{a
priori} in certain applications. The models proposed for structured
sparsity can be divided into two types. Models of the first type have
a combinatorial construction and explicitly enforce the permitted
``non-zero patterns'' \cite{baraniuk_model-based_2010,kyrillidis_combinatorial_2012}.
Greedy algorithms have been proposed for the least squares regression
with true parameters belonging to such combinatorial sparsity models
\cite{baraniuk_model-based_2010}. Models
of the second type capture sparsity patterns induced by the convex
penalty functions tailored for specific estimation problems. Typically,
such convex penalty functions are derived from convex relaxations
of the combinatorial model. For example, consistency of linear regression
with mixed $\ell_{1}$/$\ell_{2}$-norm regularization in estimation
of group sparse signals having non-overlapping groups is studied in
\cite{bach_consistency_2008}. Furthermore, a different convex penalty
to induce group sparsity with overlapping groups is proposed in \cite{jacob_group_2009}.
In \cite{bach_structured_2010}, using submodular functions and their
Lov\`{a}sz extension, a more general framework for design of convex penalties
that induce given sparsity patterns is proposed. In \cite{chandrasekaran_convex_2010}
a very general convex signal model is proposed that is generated by
a set of base signals called ``atoms''. The model can describe not
only plain and structured sparsity, but also low-rank matrices and
several other low-dimensional models. We refer readers to \cite{bach_structured_2011,duarte_structured_2011}
for extensive reviews on the estimation of signals with structured
sparsity.

In addition to linear regression problems under structured sparsity
assumptions, nonlinear statistical models have been studied in the
convex optimization framework \cite{roth_group-lasso_2008,bach_consistency_2008,jenatton_structured_2011,tewari_greedy_2011}.
For example, using the signal model introduced in \cite{chandrasekaran_convex_2010},
minimization of a convex function obeying a \emph{restricted smoothness
property} is studied in \cite{tewari_greedy_2011} where a coordinate-descent
type of algorithm is shown to converge to the minimizer at a sublinear
rate. In this formulation and other similar methods that rely on convex
relaxation one needs to choose a regularization parameter to guarantee
the desired statistical accuracy. However, choosing the appropriate
value of this parameter may be intractable. Furthermore, the convex
signal models usually provide an approximation of the ideal structures
the estimates should have, while in certain tasks such as variable
selection solutions are required to exhibit the exact structure considered.
Therefore, in such tasks, convex optimization techniques may yield
estimates that do not satisfy the desired structural properties ,
albeit accurately approximating the true parameter. These shortcomings
motivate application of combinatorial sparsity structures in nonlinear
statistical models, extending prior results such as \cite{baraniuk_model-based_2010}
that have focused exclusively on linear models.

Among the non-convex greedy algorithms, a generalization of Compressed
Sensing is considered in \cite{blumensath_compressed_2010} where
the measurement operator is a nonlinear map and the union of subspaces
is assumed as the signal model. This formulation, however, admits
only a limited class of objective functions that are described using
a norm. Furthermore, \cite{lozano_group_2011} proposes a generalization
of the Orthogonal Matching Pursuit algorithm \cite{pati_orthogonal_1993}
that is specifically designed for estimation of group sparse parameters
in Generalized Linear Models (GLMs). Also, \cite{beck_sparsity_2012}
studies the problem of minimizing a generic objective function subject
to plain sparsity constraint from the optimization perspective. Based
on certain necessary optimality conditions for the sparse minimizer,
some characterizations of sparse stationary points of the optimization
problem are proposed in \cite{beck_sparsity_2012}. Then a few iterative
algorithms are shown to converge to these stationary points, should
the objective satisfies certain conditions. In parallel to our work,
\cite{blumensath_compressed_2013} has examined a variation of this
problem, and provided similar results and guarantees. Specifically,
in \cite{blumensath_compressed_2013} the domain of the objective
function is allowed to be an infinite-dimensional Hilbert space, whereas
we only assume finite-dimensional Hilbert spaces. The sufficient conditions
introduced in\cite{blumensath_compressed_2013} is essentially equivalent
to our sufficient conditions, and both characterize the conditioning
of second-order derivatives of the objective when restricted to subspaces
of interest. The mentioned condition number controls the contraction
factor in iterations of the algorithm in both \cite{blumensath_compressed_2013}
and our work. However, to establish the convergence,  \cite{blumensath_compressed_2013}
requires the condition number to be smaller than $4/3$, which is
more stringent than our results that, depending on what is known about
the restricted second derivative, only require the condition number
to be smaller than $3$ or $3/2$. Furthermore, the accuracy of the
method in \cite{blumensath_compressed_2013} is only measured with
respect to the model-consistent minimizer and the corresponding approximation
error is expressed in terms of the value of the objective at certain
minimizers (with different model parameters).

In this paper we study the projected gradient descent method, also
a greedy algorithm, to approximate the minimizer of a cost function
subject to a model-based sparsity constraint. Our approach can be
applied to a broad set of problems, where the objective functions
are not limited to quadratic functions  or other norm-based functions
assumed in most of the previous studies. The algorithm is described
in Section \ref{sec:Problem=000026Algorithm}. The sparsity model
considered in this paper is similar to the models in \cite{baraniuk_model-based_2010}
with minor differences in the definitions. To guarantee the accuracy
of the algorithm our analysis requires the cost function to have a
Stable Model-Restricted Hessian (SMRH) as defined in Section \ref{sec:Main}.
Using this property  we can bound the distance of each iterate to
any given reference point in the considered model by the sum of two
terms. The first term, shrinks geometrically by the iterations, whereas
the second term is a fixed approximation error that depends on the
choice of the reference point. As an example, Section \ref{sec:Main}
considers the cost functions that arise in Generalized Linear Models
and discusses how the proposed sufficient condition (i.e., SMRH) can
be verified and how large the approximation error of the algorithm
is. To make precise statements on the SMRH and on the size of the
approximation error we assume some extra properties on the cost function
and/or the data distribution. Finally, we discuss and conclude in
Section \ref{sec:Discussion}.

\paragraph*{Notation.}

In the remainder of the paper we denote the positive part of a real
number $x$ by $\left(x\right)_{+}$. For a positive integer $k$,
the set $\left\{ 1,2,\ldots,k\right\} $ is denoted by $\left[k\right]$.
Vectors and matrices are denoted by boldface characters and sets by
calligraphic letters. The support set (i.e., the set of non-zero coordinates)
of a vector $\vc{x}$ is denoted by $\supp\left(\vc{x}\right)$. Restriction
of a $p$-dimensional vector $\vcg{v}$ to its entries corresponding
to an index set $\st{I}\subseteq\left[p\right]$ is denoted by $\res{\vc{v}}_{\st{I}}$.
Similarly $\mx{A}_{\st{I}}$ denotes the restriction of a matrix $\mx{A}$
to the rows enumerated by $\st{I}$. For square matrices $\mx{A}$
and $\mx{B}$ we write $\mx{B}\preccurlyeq\mx{A}$ to state that $\mx{A}-\mx{B}$
is positive semidefinite. We denote the power set of a set $\st{A}$
as $2^{\st{A}}$. For two non-empty families of sets $\st{F}_{1}$
and $\st{F}_{2}$ we write $\st{F}_{1}\Cup\st{F}_{2}$ to denote another
family of sets given by $\left\{ \st{X}_{1}\cup\st{X}_{2}\mid\st{X}_{1}\in\st{F}_{1}\text{ and }\st{X}_{2}\in\st{F}_{2}\right\} $.
Moreover, for any non-empty family of sets $\st{F}$ for conciseness
we set $\st{F}^{j}=\st{F}\Cup\ldots\Cup\st{F}$ where the operation
$\Cup$ is performed $j-1$ times. The inner product associated with
a Hilbert space $\st{H}$ is written as $\left\langle \cdot,\cdot\right\rangle $.
The norm induced by this inner product is denoted by $\norm{\cdot}$.
We use $\nabla f\left(\cdot\right)$ and $\nabla^{2}f\left(\cdot\right)$
to denote the gradient and the Hessian of a twice continuously differentiable
function $f:\st{H}\mapsto\mathbb{R}$. For an index set $\st{I}\subset\left[p\right]$
with $p=\dim\left(\st{H}\right)$, the restriction of the gradient
to the entries selected by $\st{I}$ and the restriction of the Hessian
to the entries selected by $\st{I}\times\st{I}$ are denoted by $\nabla_{\st{I}}f\left(\cdot\right)$
and $\nabla_{\st{I}}^{2}f\left(\cdot\right)$, respectively. Finally,
numerical superscripts within parentheses denote the iteration index.

\section{\label{sec:Problem=000026Algorithm}Problem Statement and Algorithm}

To formulate the problem of minimizing a cost function subject to
structured sparsity constraints, first we provide a definition of
the sparsity model. This definition is an alternative way of describing
the \emph{Combinatorial Sparse Models} in \cite{kyrillidis_combinatorial_2012}.
In comparison, our definition merely emphasizes the role of a family
of index sets as a \emph{generator} of the sparsity model.
\begin{defn}
\label{def:Model}Suppose that $p$ and $k$ are two positive integers
with $k\ll p$. Furthermore, denote by $\st{C}_{k}$ a family of some
non-empty subsets of $\left[p\right]$ that have cardinality at most
$k$. The set $\bigcup_{\st{S}\in\st{C}_{k}}2^{\st{S}}$ is called
a sparsity model of order $k$ generated by $\st{C}_{k}$ and denoted
by $\st{M}\left(\st{C}_{k}\right)$.\end{defn}
\begin{rem}
Note that if a set $\st{S}\in\st{C}_{k}$ is a subset of another set
in $\st{C}_{k}$, then the same sparsity model can still be generated
after removing $\st{S}$ from $\st{C}_{k}$ (i.e., $\st{M}\left(\st{C}_{k}\right)=\st{M}\left(\st{C}_{k}\backslash\left\{ \st{S}\right\} \right)$).
Thus, we can assume that there is no pair of distinct sets in $\st{C}_{k}$
that one is a subset of the other.
\end{rem}
In this paper we aim to approximate the solution to the optimization
problem 
\begin{alignat}{1}
\argmin_{\vcg{\uptheta}\in\st{H}}f\left(\vcg{\uptheta}\right)\quad\text{s.t. }\supp\left(\vcg{\uptheta}\right)\in\st{M}\left(\st{C}_{k}\right),\label{eq:ModelConsOpt}
\end{alignat}
 where $f:\st{H}\mapsto\mathbb{R}$ is a cost function with $\st{H}$
being a $p$-dimensional real Hilbert space, and $\st{M}\left(\st{C}_{k}\right)$
a given sparsity model described by Def. \ref{def:Model}.
\begin{rem}
In the context of statistical estimation, the cost function $f\left(\cdot\right)$
is usually the empirical loss associated with some observations generated
by an underlying true parameter $\vcg{\uptheta}^{\star}$. In these
problems, it is more desired to estimate $\vcg{\uptheta}^{\star}$
as it describes the data. The analysis presented in this paper allows
evaluating the approximation error of the proposed algorithm with
respect to any parameter vector in the considered sparsity model including any solution $\widehat{\vcg{\uptheta}}$ to (\ref{eq:ModelConsOpt}) as well as the statistical truth $\vcg{\uptheta}^{\star}$. However,
the approximation error with respect to $\vcg{\uptheta}^{\star}$
can be simplified and interpreted to a greater extent. We elaborate
more on this in Section \ref{sec:Main}.
\end{rem}
To approximate a solution $\widehat{\vcg{\uptheta}}$ to (\ref{eq:ModelConsOpt})
we use a \emph{projected gradient descent} method summarized in Algorithm
\ref{alg:GDMS}. The only difference between Algorithm \ref{alg:GDMS}
and standard projected gradient descent methods studied in convex
optimization literature is that the projection, in line \ref{lin:Proj},
is performed onto the generally non-convex set $\st{M}\left(\st{C}_{k}\right)$.
The projection operator $\mathrm{P}_{\st{C}_{k},r}:\st{H}\mapsto\st{H}$
at any given point $\vcg{\uptheta}_{0}\in\st{H}$ is defined as a
solution to 
\begin{alignat}{1}
\argmin_{\vcg{\uptheta}\in\st{H}}\norm{\vcg{\uptheta}-\vcg{\uptheta}_{0}}\quad\text{s.t. }\supp\left(\vcg{\uptheta}\right)\in\st{M}\left(\st{C}_{k}\right)\text{ and }\norm{\vcg{\uptheta}}\leq r.\label{eq:Proj}
\end{alignat}

{\centering
\begin{algorithm2e}[t]
	\DontPrintSemicolon
	\caption{Gradient Descent with Model Sparsity Constraint\label{alg:GDMS}}
	\SetKwInOut{Input}{input}
	\SetKwInOut{Output}{output}
	\Input{$\st{C}_k$, the family of possible supports,\\ $r$, the radius of feasible set}	
	$i\longleftarrow 0$ , $\vcg{\uptheta}\itr{i}\longleftarrow \vc{0}$\;
	\Repeat{halting condition holds}{
\nl	Choose step-size $\eta\itr{i}>0$\;
\nl\label{lin:Descent}	$\vcg{\upchi}\itr{i}\longleftarrow\vcg{\uptheta}\itr{i}-\eta\itr{i}\nabla f\left(\vcg{\uptheta}\itr{i}\right)$\;
\nl\label{lin:Proj}	$\vcg{\uptheta}\itr{i+1}\longleftarrow\mathrm{P}_{\st{C}_k,r}\left[\vcg{\upchi}\itr{i}\right]$\;
\nl	$i\longleftarrow i+1$\;
	}
	\Return{$\vcg{\uptheta}\itr{i}$}
\end{algorithm2e}}

\begin{rem}
One may also question the necessity of the constraint $\norm{\vcg{\uptheta}}\leq r$
in (\ref{eq:Proj}). As discussed later in Section \ref{sec:Example},
in statistical estimation problems where the cost function is not
quadratic the sufficient condition we rely on cannot be guaranteed
to hold unless the iterates and the true parameter lie in a bounded
set. This shortcoming is typical for convergence proofs that use similar
types of conditions (cf. \cite{bunea_honest_2008,negahban_unified_2009,agarwal_fast_2010,kakade_learning_2010}).
Finally, the exact projection onto the sparsity model $\st{M}\left(\st{C}_{k}\right)$
might not be tractable. 

Existence and complexity of algorithms that find the desired exact
or approximate projections, disregarding the length constraint in
(\ref{eq:Proj}) (i.e., $\mathrm{\mathrm{P}_{\st{C}_{k},+\infty}\left[\cdot\right]}$),
are studied in \cite{kyrillidis_combinatorial_2012}
for several interesting sparsity models. Furthermore, such projections
are known and tractable for signals with block-sparse support or support
that satisfies a tree model \cite{baraniuk_model-based_2010}. Often,
one may also desire to show that accuracy can be guaranteed even using
an inexact projection operator, at the cost of an extra error term.
For example, it is recently shown in \cite{hegde_approximation-tolerant_2014}
how to extend the framework of model-based compressed sensing to admit
inexact projections by assuming ``head'' and ``tail'' oracles
for the projections. In other cases, such as the co-sparse analysis
signal model, such projections are assumed but not theoretically backed\cite{B_TIT_11,DNW_TIT13,GNEGD_LAA14}.
Also, in the general case where $r<+\infty$, one can derive a projection
$\mathrm{P}_{\st{C}_{k},r}\left[\vcg{\uptheta}\right]$ from $\mathrm{P}_{\st{C}_{k},+\infty}\left[\vcg{\uptheta}\right]$
(see Lemma \ref{lem:BoundProj} in the Appendix). 

It is straightforward to generalize the guarantees in this paper to
cases where only approximate projection is tractable. However, we
do not attempt it here; our focus is to study the algorithm when the
cost function is not necessarily quadratic. Instead, we apply the
results to statistical estimation problems with non-linear models
and we derive bounds on the statistical error of the estimate.
\end{rem}

\section{\label{sec:Main}Theoretical Analysis}

\subsection{Stable Model-Restricted Hessian}

In order to demonstrate accuracy of estimates obtained using Algorithm
\ref{alg:GDMS} we require a variant of the \emph{Stable Restricted
Hessian} (SRH) condition proposed in \cite{bahmani_GraSP_2012} to
hold. The SRH condition basically characterizes cost functions that
have bounded curvature over canonical sparse subspaces. In this paper
we require this condition to hold merely for the signals that belong
to the considered model. Furthermore, we explicitly bound the length
of the vectors at which the condition should hold. As will be discussed
later, this restriction is necessary in general for non-quadratic
cost functions. The condition we rely on, the Stable Model-Restricted
Hessian (SMRH), can be formally defined as follows.
\begin{defn}
\label{def:SMRH}Let $f:\st{H}\mapsto\mathbb{R}$ be a twice continuously
differentiable function. Furthermore, let $\alpha_{\st{C}_{k}}$ and
$\beta_{\st{C}_{k}}$ be in turn the largest and smallest positive
real numbers such that 
\begin{alignat}{2}
\beta_{\st{C}_{k}}\norm{\vcg{\Delta}}^{2} & \leq\left\langle \vcg{\Delta},\nabla^{2}f\left(\vcg{\uptheta}\right)\vcg{\Delta}\right\rangle  & \leq\alpha_{\st{C}_{k}}\norm{\vcg{\Delta}}^{2},\label{eq:SRH}
\end{alignat}
 holds for all $\vcg{\Delta}$ and $\vcg{\uptheta}$ such that $\supp\left(\vcg{\Delta}\right)\cup\supp\left(\vcg{\uptheta}\right)\in\st{M}\left(\st{C}_{k}\right)$
and $\norm{\vcg{\uptheta}}\leq r$. Then $f$ is said to have a Stable
Model-Restricted Hessian (SMRH) with respect to the model $\st{M}\left(\st{C}_{k}\right)$
with constants $\alpha_{\st{C}_{k}}$ and $\beta_{\st{C}_{k}}$ in
a sphere of radius $r>0$, or in short ($\alpha_{\st{C}_{k}}$, $\beta_{\st{C}_{k}}$,
$r$)-SMRH. The conditioning for this SMRH is also denoted by $\mu_{\st{C}_{k}}:=\alpha_{\st{C}_{k}}/\beta_{\st{C}_{k}}$.
\end{defn}
Consider the special case of $f\left(\vcg{\uptheta}\right)=\frac{1}{2}\norm{\mx{X}\vcg{\uptheta}-\vc{y}}_{2}^{2}$
as in compressed sensing or sparse linear regression. It is straightforward
to see that the SMRH effectively reduces to the model-based restricted
isometry property by setting $\nabla^{2}f\left(\vcg{\uptheta}\right)=\mx{X}\tran\mx{X}$
in the definition of the SMRH. The model-based restricted isometry
constant $\delta_{\st{C}_{k}}$ and the SMRH constants are related
in this special case via $\beta_{\st{C}_{k}}\geq1-\delta_{\st{C}_{k}}$,
$\alpha_{\st{C}_{k}}\leq1+\delta_{\st{C}_{k}}$, and $\mu_{\st{C}_{k}}\leq\left(1+\delta_{\st{C}_{k}}\right)/\left(1-\delta_{\st{C}_{k}}\right)$.
\begin{rem}
Typically in parametric estimation problems a sample loss function
$\ell\left(\vcg{\uptheta},\vc{x},y\right)$ is associated with the covariate-response
pair $\left(\vc{x},y\right)$ and a parameter $\vcg{\uptheta}$. Given
$n$ iid observations the empirical loss is formulated as $\widehat{L}_{n}\left(\vcg{\uptheta}\right)=\frac{1}{n}\sum_{i=1}^{n}\ell\left(\vcg{\uptheta},\vc{x}_{i},y_{i}\right)$.
The estimator under study is the minimizer of the empirical loss,
perhaps considering an extra regularization or constraint for the
parameter $\vcg{\uptheta}$. To prove accuracy of sparse estimation
algorithms it is often required that the cost function is strongly
convex/smooth over a restricted set of directions as a sufficient
condition. It is known, however, that $\widehat{L}_{n}\left(\vcg{\uptheta}\right)$
as an empirical process is a good approximation of the expected loss
$L\left(\vcg{\uptheta}\right)=\mathbb{E}\left[\ell\left(\vcg{\uptheta},\vc{x},y\right)\right]$
(see \cite{van_de_geer_empirical_2000} and \cite[Chapter 5]{vapnik_statistical_1998}).
If $L\left(\vcg{\uptheta}\right)$ does not satisfy the desired restricted
strong convexity/smoothness conditions globally for all choices of
the true parameter $\vcg{\uptheta}^{\star}$ that have the structured
sparsity, then in general $\widehat{L}_{n}\left(\vcg{\uptheta}\right)$
does not satisfy the desired conditions globally, either. Thus, as
also assumed in the prior work either explicitly \cite{bunea_honest_2008}
or implicitly \cite{negahban_unified_2009,kakade_learning_2010,agarwal_fast_2010},
for a generic sample loss it is only possible to guarantee these types
of sufficient conditions if the set of valid vectors $\vcg{\uptheta}^{\star}$
are further restricted, e.g., by bounding their length. This is the
motivation behind the restriction imposed on the length of $\vcg{\uptheta}$
in Def. \ref{def:SMRH}. Of course, if the true parameter violates
this restriction we may incur an estimation bias as quantified in
Theorem \ref{thm:IterInvar}.
\end{rem}
The SMRH is similar to other conditions such as SRH \cite{bahmani_GraSP_2012}
and various forms of Restricted Strong Convexity/Smoothness (RSC/RSS)
(e.g., \cite{negahban_unified_2009} and \cite{blumensath_compressed_2013}):
they all impose quadratic bounds on the second derivative of the objective
function when restricted to sparse or model-sparse vectors. However,
there are some subtle differences. The SRH is defined for plain sparse
vectors and its quadratic bounds are defined locally. For the SMRH,
however, the fact that the boundedness is incorporated in the signal
model allowed us to define the quadratic bounds globally. The RSC
defined in \cite{blumensath_compressed_2013} is more general than
the SMRH since it assumes infinite-dimensional Hilbert spaces as the
domain of the function, whereas in SMRH we consider function defined
over finite-dimensional Hilbert spaces. However, the accuracy analysis
of \cite{blumensath_compressed_2013} guarantees convergence of the
projected gradient descent $\mu_{\st{C}_{k}^{3}}<4/3$, whereas ,
as will be shown by Corollary \ref{cor:FixedStep}, we can prove convergence
of the algorithm for $\mu_{\st{C}_{k}^{3}}<3/2$ or even $\mu_{\st{C}_{k}^{3}}<3$.

\subsection{Accuracy Guarantee}

Recall that in our notation $\st{C}_{k}^{2}=\st{C}_{k}\Cup\st{C}_{k}$
and $\st{C}_{k}^{3}=\st{C}_{k}\Cup\st{C}_{k}\Cup\st{C}_{k}$. Intuitively,
$\st{C}_{k}^{2}$ and $\st{C}_{k}^{3}$ can describe all possible
support sets of the sum of two or three vectors in $\st{M}\left(\st{C}_{k}\right)$,
respectively. Using the notion of SMRH we can now state the main theorem.
\begin{thm}
\label{thm:IterInvar}Consider the sparsity model $\st{M}\left(\st{C}_{k}\right)$
for some $k\in\mathbb{N}$ and a cost function $f:\st{H}\mapsto\mathbb{R}$
that satisfies the $\left(\alpha_{\st{C}_{k}^{3}},\beta_{\st{C}_{k}^{3}},r\right)$-SMRH
condition as in (\ref{eq:SRH}) with $\mu_{\st{C}_{k}^{3}}:=\alpha_{\st{C}_{k}^{3}}/\beta_{\st{C}_{k}^{3}}$.
If $\eta^{\star}=2/\left(\alpha_{\st{C}_{k}^{3}}+\beta_{\st{C}_{k}^{3}}\right)$
then for any $\overline{\vcg{\uptheta}}\in\st{M}\left(\st{C}_{k}\right)$
with $\norm{\overline{\vcg{\uptheta}}}\leq r$ the iterates of Algorithm
\ref{alg:GDMS} obey 
\begin{alignat}{1}
\norm{\vcg{\uptheta}\itr{i+1}-\overline{\vcg{\uptheta}}} & \leq2\gamma\itr{i}\norm{\vcg{\uptheta}\itr{i}-\overline{\vcg{\uptheta}}}+2\eta\itr{i}\norm{\nabla_{\overline{\st{I}}}f\left(\overline{\vcg{\uptheta}}\right)},\label{eq:MainTheorem}
\end{alignat}
 where $\gamma\itr{i}=\frac{\eta\itr{i}}{\eta^{\star}}\frac{\mu_{\st{C}_{k}^{3}}-1}{\mu_{\st{C}_{k}^{3}}+1}+\left|\frac{\eta\itr{i}}{\eta^{\star}}-1\right|$
and $\overline{\st{I}}=\supp\left(\mathrm{P}_{\st{C}_{k}^{2},r}\left[\nabla f\left(\overline{\vcg{\uptheta}}\right)\right]\right)$.
\end{thm}
Theorem \ref{thm:IterInvar} can be used to localize an ``attractor set'' of the iterates $\vcg{\uptheta}\itr{i}$ with respect to the desired reference point $\overline{\vcg{\uptheta}}$. In particular, if $2\gamma\itr{i}<1$ the bound \eqref{eq:MainTheorem} guarantees an approximate contraction which can be used recursively as in Corollary \ref{cor:FixedStep} below. Ideally, the iterates eventually fall within a relatively small ball centered at the desired $\overline{\vcg{\uptheta}}$. We refer to  the radius of this ball as the approximation error. Theorem \ref{thm:IterInvar} helps to bound the approximation error in terms of $\nabla f\left(\overline{\vcg{\uptheta}}\right)$. For instance, if at a model-sparse minimizer obtained by \eqref{eq:ModelConsOpt} the gradient of the objective (restricted to indices $\overline{\st{I}}$ ) has a small $\ell_{2}$-norm then the iterates can provide accurate estimates of the minimizer. In particular, if the restricted gradient vanishes at a model-sparse minimizer, the approximation error with respect to that point would be zero, i.e., the iterates converge provided that $2\gamma\itr{i}<1$ is guaranteed. This scenario can occur when the objective has multiple stationary points only one of which is within the model, a typical case in high-dimensional estimation problems. For example, classical guarantees for noiseless sparse linear regression provide exact recovery if the signal is exactly sparse, but only an approximation if the signal is approximately sparse. While \eqref{eq:MainTheorem} still holds if  the gradient is not restricted to $\overline{\st{I}}$, the obtained approximation error might not be sufficiently tight. For example, applying \eqref{eq:MainTheorem}  in statistical estimation problems with $\nabla f(\vcg{\uptheta})$  replacing $\nabla_{\overline{\st{I}}}f\left(\overline{\vcg{\uptheta}}\right)$ would yield loose estimation error bounds that grow with the ambient dimension rather than the sparsity of the target parameter.
\begin{rem}
One should choose the step size to achieve a contraction factor $2\gamma\itr{i}$
that is as small as possible. Straightforward algebra shows that the
constant step-size $\eta\itr{i}=\eta^{\star}$ is optimal, but this
choice may not be practical as the constants $\alpha_{\st{C}_{k}^{3}}$
and $\beta_{\st{C}_{k}^{3}}$ might not be known. Instead, we can
always choose the step-size such that $1/\alpha_{\st{C}_{k}^{3}}\leq\eta\itr{i}\leq1/\beta_{\st{C}_{k}^{3}}$
provided that the cost function obeys the SMRH condition. It suffices
to set $\eta\itr{i}=1/\left\langle \vcg{\Delta},\nabla^{2}f\left(\vcg{\uptheta}\right)\vcg{\Delta}\right\rangle $
for some $\vcg{\Delta}$,$\vcg{\uptheta}\in\st{H}$ that obeys $\supp\left(\vcg{\Delta}\right)\cup\supp\left(\vcg{\uptheta}\right)\in\st{M}\left(\st{C}_{k}^{3}\right)$.
For this choice of $\eta\itr{i}$, we have $\gamma\itr{i}\leq\mu_{\st{C}_{k}^{3}}-1$.\end{rem}
\begin{cor}
\label{cor:FixedStep}A fixed step-size $\eta>0$ corresponds to a
fixed contraction coefficient $\gamma=\frac{\eta}{\eta^{\star}}\frac{\mu_{\st{C}_{k}^{3}}-1}{\mu_{\st{C}_{k}^{3}}+1}+\left|\frac{\eta}{\eta^{\star}}-1\right|$.
In this case, assuming that $2\gamma\neq1$, the $i$-th iterate of
Algorithm \ref{alg:GDMS} satisfies 
\begin{alignat}{1}
\norm{\vcg{\uptheta}\itr{i}-\overline{\vcg{\uptheta}}} & \leq\left(2\gamma\right)^{i}\norm{\overline{\vcg{\uptheta}}}+2\eta\frac{1-\left(2\gamma\right)^{i}}{1-2\gamma}\norm{\nabla_{\overline{\st{I}}}f\left(\overline{\vcg{\uptheta}}\right)}.\label{eq:FixedStep}
\end{alignat}
 In particular, 
\begin{enumerate}
\item [(i)]if $\mu_{\st{C}_{k}^{3}}<3$ and $\eta=\eta^{\star}=2/\left(\alpha_{\st{C}_{k}^{3}}+\beta_{\st{C}_{k}^{3}}\right)$,
or
\item [(ii)]if $\mu_{\st{C}_{k}^{3}}<\frac{3}{2}$ and $\eta\in\left[1/\alpha_{\st{C}_{k}^{3}},1/\beta_{\st{C}_{k}^{3}}\right]$, 
\end{enumerate}
the distance of the iterates and $\overline{\vcg{\uptheta}}$ shrinks
up to an approximation error bounded above by $\frac{2\eta}{1-2\gamma}\norm{\nabla_{\overline{\st{I}}}f\left(\overline{\vcg{\uptheta}}\right)}$
with contraction factor $2\gamma<1$.\end{cor}
\begin{proof}
Applying (\ref{eq:MainTheorem}) recursively under the assumptions
of the corollary and using the identity $\sum_{j=0}^{i-1}\left(2\gamma\right)^{j}=\frac{1-\left(2\gamma\right)^{i}}{1-2\gamma}$
proves (\ref{eq:FixedStep}). In the first case, if $\mu_{\st{C}_{k}^{3}}<3$
and $\eta=\eta^{\star}=2/\left(\alpha_{\st{C}_{k}^{3}}+\beta_{\st{C}_{k}^{3}}\right)$
we have $2\gamma<1$ by definition of $\gamma$. In the second case,
one can deduce from $\eta\in\left[1/\alpha_{\st{C}_{k}^{3}},1/\beta_{\st{C}_{k}^{3}}\right]$
that $\left|\eta/\eta^{\star}-1\right|\leq\frac{\mu_{\st{C}_{k}^{3}}-1}{2}$
and $\eta/\eta^{\star}\leq\frac{\mu_{\st{C}_{k}^{3}}+1}{2}$ where
equalities are attained simultaneously at $\eta=1/\beta_{\st{C}_{k}^{3}}$.
Therefore, $\gamma\leq\mu_{\st{C}_{k}^{3}}-1<1/2$ and thus $2\gamma<1$.
Finally, in both cases it immediately follows from (\ref{eq:FixedStep})
that the approximation error converges to $\frac{2\eta}{1-2\gamma}\norm{\nabla_{\overline{\st{I}}}f\left(\overline{\vcg{\uptheta}}\right)}$
from below as $i\to+\infty$.
\end{proof}

\section{\label{sec:Example}Application in Generalized Linear Models}

Generalized Linear Models (GLMs) are among the most commonly used
models for parametric estimation in variety of applications \cite{dobson_introduction_2008}.
Linear, logistic, Poisson, and gamma models used in corresponding
regression problems all belong to the family of GLMs. Given a covariate
vector $\vc{x}\in\st{X}\subseteq\mathbb{R}^{p}$ and a true parameter
$\vcg{\uptheta}^{\star}\in\mathbb{R}^{p}$, the response variable
$y\in\st{Y}\subseteq\mathbb{R}$ in canonical GLMs is assumed to follow
an exponential family conditional distribution: $y\mid\vc{x};\vcg{\uptheta}^{\star}\sim Z\left(y\right)\exp\left(y\left\langle \vc{x},\vcg{\uptheta}^{\star}\right\rangle -\psi\left(\left\langle \vc{x},\vcg{\uptheta}^{\star}\right\rangle \right)\right),$
where $Z\left(y\right)$ is a positive function, and $\psi:\mathbb{R}\mapsto\mathbb{R}$
is the \emph{log-partition function} that satisfies $\psi\left(t\right)=\log\int_{\st{Y}}Z\left(y\right)\exp\left(ty\right)\dx y$
for all $t\in\mathbb{R}$. Examples of the log-partition function
include but are not limited to $\psi_{\text{lin}}\left(t\right)=t^{2}/2\sigma^{2}$,
$\psi_{\log}\left(t\right)=\log\left(1+\exp\left(t\right)\right)$,
and $\psi_{\text{Pois}}\left(t\right)=\exp\left(t\right)$ corresponding
to linear, logistic, and Poisson models, respectively.%

Suppose that $n$ iid covariate-response pairs $\left\{ \left(\vc{x}_{i},y_{i}\right)\right\} _{i=1}^{n}$
are observed. In the Maximum Likelihood Estimation (MLE) framework
the negative log likelihood is used as a measure of the discrepancy
between the true parameter $\vcg{\uptheta}^{\star}$ and an estimate
$\vcg{\uptheta}$ based on the observations. Formally, the average
of negative log likelihoods is considered as the empirical loss 
\begin{alignat*}{1}
f\left(\vcg{\uptheta}\right) & =\frac{1}{n}\sum_{i=1}^{n}\psi\left(\left\langle \vc{x}_{i},\vcg{\uptheta}\right\rangle \right)-y_{i}\left\langle \vc{x}_{i},\vcg{\uptheta}\right\rangle ,
\end{alignat*}
 and the MLE is performed by minimizing $f\left(\vcg{\uptheta}\right)$
over the set of feasible $\vcg{\uptheta}$. The constants $c$ and
$Z$ that appear in the distribution are disregarded as they have
no effect in the outcome.

\subsection{Verifying SMRH for GLMs}

Assuming that $\psi\left(\cdot\right)$ is twice continuously differentiable,
the Hessian of $f\left(\cdot\right)$ is equal to 
\begin{alignat*}{1}
\nabla^{2}f\left(\vcg{\uptheta}\right) & =\frac{1}{n}\sum_{i=1}^{n}\psi''\left(\left\langle \vc{x}_{i},\vcg{\uptheta}\right\rangle \right)\vc{x}_{i}\vc{x}_{i}\tran.
\end{alignat*}
 Under the assumptions for GLMs, it can be shown that $\psi''\left(\cdot\right)$
is non-negative (i.e., $\psi\left(\cdot\right)$ is convex). For a
given sparsity model generated by $\st{C}_{k}$ let $\st{S}$ be an
arbitrary support set in $\st{C}_{k}$ and suppose that $\supp\left(\vcg{\uptheta}\right)\subseteq\st{S}$
and $\norm{\theta}\leq r$ . Furthermore, define 
\begin{alignat*}{3}
D_{\psi,r}\left(u\right) & :=\max_{t\in\left[-r,r\right]}\psi''\left(tu\right) & \quad\text{and}\quad & d_{\psi,r}\left(u\right) & :=\min_{t\in\left[-r,r\right]}\psi''\left(tu\right).
\end{alignat*}
Using the Cauchy-Schwarz inequality we have $\left|\left\langle \vc{x}_{i},\vcg{\uptheta}\right\rangle \right|\leq r\norm{\res{\vc{x}_{i}}_{\st{S}}}$
which implies 
\begin{alignat*}{2}
\frac{1}{n}\sum_{i=1}^{n}d_{\psi,r}\left(\norm{\res{\vc{x}_{i}}_{\st{S}}}\right)\res{\vc{x}_{i}}_{\st{S}}\res{\vc{x}_{i}}_{\st{S}}\tran & \preccurlyeq\nabla_{\st{S}}^{2}f\left(\vcg{\uptheta}\right) & \preccurlyeq\frac{1}{n}\sum_{i=1}^{n}D_{\psi,r}\left(\norm{\res{\vc{x}_{i}}_{\st{S}}}\right)\res{\vc{x}_{i}}_{\st{S}}\res{\vc{x}_{i}}_{\st{S}}\tran.
\end{alignat*}
These matrix inequalities are precursors of (\ref{eq:SRH}). Imposing
further restriction on the distribution of the covariate vectors $\left\{ \vc{x}_{i}\right\} _{i=1}^{n}$
allows application of the results from random matrix theory regarding
the extreme eigenvalues of random matrices (see e.g., \cite{tropp_user-friendly_2011}
and \cite{hsu_tail_2011}). 

For example, in the logistic model where $\psi\equiv\psi_{\log}$
we can show that $D_{\psi,r}\left(u\right)=\frac{1}{4}$ and $d_{\psi,r}\left(u\right)=\frac{1}{4}\mathrm{sech}^{2}\left(\frac{ru}{2}\right)$.
Assuming that the covariate vectors are iid instances of a random
vectors whose length almost surely bounded by one, we obtain $d_{\psi,r}\left(u\right)\geq\frac{1}{4}\mathrm{sech}^{2}\left(\frac{r}{2}\right)$.
Using the matrix Chernoff inequality \cite{tropp_user-friendly_2011}
the extreme eigenvalues of $\frac{1}{n}\mx{X}_{\st{S}}\mx{X}_{\st{S}}\tran$
can be bounded with probability $1-\exp\left(\log k-Cn\right)$ for
some constant $C>0$ (see \cite{bahmani_GraSP_2012} for detailed
derivations). Using these results and taking the union bound over
all $\st{S}\in\st{C}_{k}$ we obtain bounds for the extreme eigenvalues
of $\nabla_{\st{S}}^{2}f\left(\vcg{\uptheta}\right)$ that hold uniformly
for all sets $\st{S}\in\st{C}_{k}$ with probability $1-\exp\left(\log\left(k\left|\st{C}_{k}\right|\right)-Cn\right)$.
Thus (\ref{eq:SRH}) may hold if $n=O\left(\log\left(k\left|\st{C}_{k}\right|\right)\right)$.

\subsection{Approximation Error for GLMs}

Suppose that the approximation error is measured with respect to $\vcg{\uptheta}^{\perp}=\mathrm{P}_{\st{C}_{k},r}\left[\vcg{\uptheta}^{\star}\right]$
where $\vcg{\uptheta}^{\star}$ is the statistical truth in the considered
GLM. It is desirable to further simplify the approximation error bound
provided in Corollary \ref{cor:FixedStep} which is related to the
statistical precision of the estimation problem. The corollary provides
an approximation error that is proportional to $\norm{\nabla_{\st{T}}f\left(\vcg{\uptheta}^{\perp}\right)}$
where $\mbox{\ensuremath{\st{T}=\supp\left(\mathrm{P}_{\st{C}_{k}^{2},r}\left[\nabla f\left(\vcg{\uptheta}^{\perp}\right)\right]\right)}}$
. We can write 
\begin{alignat*}{1}
\nabla_{\st{T}}f\left(\vcg{\uptheta}^{\perp}\right) & =\frac{1}{n}\sum_{i=1}^{n}\left(\psi'\left(\left\langle \vc{x}_{i},\vcg{\uptheta}^{\perp}\right\rangle \right)-y_{i}\right)\res{\vc{x}_{i}}_{\st{T}},
\end{alignat*}
 which yields $\norm{\nabla_{\st{T}}f\left(\vcg{\uptheta}^{\perp}\right)}=\norm{\mx{X}_{\st{T}}\vc{z}}$
where $\mx{X}=\frac{1}{\sqrt{n}}\left[\begin{array}{cccc}
\vc{x}_{1} & \vc{x}_{2} & \cdots & \vc{x}_{n}\end{array}\right]$ and $\res{\vc{z}}_{\left\{ i\right\} }=z_{i}=\frac{\psi'\left(\left\langle \vc{x}_{i},\vcg{\uptheta}^{\perp}\right\rangle \right)-y_{i}}{\sqrt{n}}$.
Therefore, 
\begin{alignat*}{1}
\norm{\nabla_{\st{T}}f\left(\vcg{\uptheta}^{\perp}\right)}^{2} & \leq\norm{\mx{X}_{\st{T}}}_{\mathrm{op}}^{2}\norm{\vc{z}}^{2},
\end{alignat*}
 where $\norm{\cdot}_{\mathrm{op}}$ denotes the operator norm. Again
using random matrix theory one can find an upper bound for $\norm{\mx{X}_{\st{I}}}_{\mathrm{op}}$
that holds uniformly for any $\st{I}\in\st{C}_{k}^{2}$ and in particular
for $\st{I}=\st{T}$. Henceforth, $W>0$ is used to denote this upper
bound.

The second term in the bound can be written as 
\begin{alignat*}{1}
\norm{\vc{z}}^{2} & =\frac{1}{n}\sum_{i=1}^{n}\left(\psi'\left(\left\langle \vc{x}_{i},\vcg{\uptheta}^{\perp}\right\rangle \right)-y_{i}\right)^{2}.
\end{alignat*}
 To further simplify this term we need to make assumptions about the
log-partition function $\psi\left(\cdot\right)$ and/or the distribution
of the covariate-response pair $\left(\vc{x},y\right)$. For instance,
if $\psi'\left(\cdot\right)$ and the response variable $y$ are bounded,
as in the logistic model, then Hoeffding's inequality implies that
for some small $\epsilon>0$ we have $\norm{\vc{z}}^{2}\leq\mathbb{E}\left[\left(\psi'\left(\left\langle \vc{x},\vcg{\uptheta}^{\perp}\right\rangle \right)-y\right)^{2}\right]+\epsilon$
with probability at least $1-\exp\left(-O\left(\epsilon^{2}n\right)\right)$.
Since in GLMs the true parameter $\vcg{\uptheta}^{\star}$ is the
minimizer of the expected loss $\mathbb{E}\left[\psi\left(\left\langle \vc{x},\vcg{\uptheta}\right\rangle \right)-y\left\langle \vc{x},\vcg{\uptheta}\right\rangle \mid\vc{x}\right]$
we deduce that $\mathbb{E}\left[\psi'\left(\left\langle \vc{x},\vcg{\uptheta}^{\star}\right\rangle \right)-y\mid\vc{x}\right]=0$
and hence $\mathbb{E}\left[\psi'\left(\left\langle \vc{x},\vcg{\uptheta}^{\star}\right\rangle \right)-y\right]=0$.
Therefore, 
\begin{alignat*}{1}
\norm{\vc{z}}^{2} & \leq\mathbb{E}\left[\mathbb{E}\left[\left(\psi'\left(\left\langle \vc{x},\vcg{\uptheta}^{\perp}\right\rangle \right)-\psi'\left(\left\langle \vc{x},\vcg{\uptheta}^{\star}\right\rangle \right)\right.\right.\right.\left.\left.+\psi'\left(\left\langle \vc{x},\vcg{\uptheta}^{\star}\right\rangle \right)-y\Bigr)^{2}\mid\vc{x}\right]\right]+\epsilon\\
 & \leq\mathbb{E}\left[\left(\psi'\left(\left\langle \vc{x},\vcg{\uptheta}^{\perp}\right\rangle \right)-\psi'\left(\left\langle \vc{x},\vcg{\uptheta}^{\star}\right\rangle \right)\right)^{2}\right]+\mathbb{E}\left[\left(\psi'\left(\left\langle \vc{x},\vcg{\uptheta}^{\star}\right\rangle \right)-y\right)^{2}\right]+\epsilon.\\
 & =\underbrace{\mathbb{E}\left[\left(\psi'\left(\left\langle \vc{x},\vcg{\uptheta}^{\perp}\right\rangle \right)-\psi'\left(\left\langle \vc{x},\vcg{\uptheta}^{\star}\right\rangle \right)\right)^{2}\right]}_{\delta_{1}}+\underbrace{\mathrm{var}\left(\psi'\left(\left\langle \vc{x},\vcg{\uptheta}^{\star}\right\rangle \right)-y\right)+\epsilon}_{\sigma_{\text{{stat}}}^{2}}.
\end{alignat*}
 Then it follows from Corollary \ref{cor:FixedStep} and the fact
that $\norm{\res{\mx{X}}_{\st{I}}}_{\mathrm{op}}\leq W$ that 
\begin{alignat*}{1}
\norm{\vcg{\uptheta}\itr{i}-\vcg{\uptheta}^{\star}} & \leq\norm{\vcg{\uptheta}\itr{i}-\vcg{\uptheta}^{\perp}}+\underbrace{\norm{\vcg{\uptheta}^{\perp}-\vcg{\uptheta}^{\star}}}_{\delta_{2}}\\
 & \leq\left(2\gamma\right)^{i}\norm{\vcg{\uptheta}^{\perp}}+\frac{2\eta W}{1-2\gamma}\sigma_{\text{stat}}^{2}+\frac{2\eta W}{1-2\gamma}\delta_{1}+\delta_{2}.
\end{alignat*}
We see the total approximation error is comprised of two parts. The
first part is due to statistical error that is given by $\frac{2\eta W}{1-2\gamma}\sigma_{\text{stat}}^{2}$,
and $\frac{2\eta W}{1-2\gamma}\delta_{1}+\delta_{2}$ is the second
part of the error due to the bias that occurs because of an infeasible
true parameter. The bias vanishes if the true parameter lies in the
considered bounded sparsity model (i.e., $\vcg{\uptheta}^{\star}=\mathrm{P}_{\st{C}_{k},r}\left[\vcg{\uptheta}^{\star}\right]$).

\section{\label{sec:Discussion}Conclusion}

We studied the projected gradient descent method for minimization
of a real valued cost function defined over a finite-dimensional Hilbert
space, under structured sparsity constraints. Using previously known
combinatorial sparsity models, we define a sufficient condition for
accuracy of the algorithm, the SMRH. Under this condition the algorithm
produces an approximation for the desired optimum at a linear rate.
Unlike the previous results on greedy-type methods that merely have
focused on linear statistical models, our algorithm applies to a broader
family of estimation problems. To provide an example, we examined
application of the algorithm in estimation with GLMs and showed how
the SMRH can be verified for these models. The approximation error
can also be bounded by statistical precision and the potential bias.
An interesting follow-up problem is to find whether the approximation
error can be improved and the derived error is merely a by-product
of requiring some form of restricted strong convexity through SMRH.

\appendix
[Proofs]
\begin{lem}
\label{lem:BasicIneq}Suppose that $f$ is a twice differentiable
function that satisfies (\ref{eq:SRH}) for a given $\vcg{\uptheta}$
and all $\vcg{\Delta}$ such that $\supp\left(\vcg{\Delta}\right)\cup\supp\left(\vcg{\uptheta}\right)\in\st{M}\left(\st{C}_{k}\right)$.
Then we have 
\begin{alignat*}{1}
\left|\left\langle \vc{u},\vc{v}\right\rangle -\eta\left\langle \vc{u},\nabla^{2}f\left(\vcg{\uptheta}\right)\vc{v}\right\rangle \right| & \leq\left(\eta\frac{\alpha_{\st{C}_{k}}-\beta_{\st{C}_{k}}}{2}+\left|\eta\frac{\alpha_{\st{C}_{k}}+\beta_{\st{C}_{k}}}{2}-1\right|\right)\norm{\vc{u}}\norm{\vc{v}},
\end{alignat*}
 for all $\eta>0$ and $\vc{u},\vc{v}\in\st{H}$ such that $\supp\left(\vc{u}\pm\vc{v}\right)\cup\supp\left(\vcg{\uptheta}\right)\in\st{M}\left(\st{C}_{k}\right)$. \end{lem}
\begin{proof}
We first the prove the lemma for unit-norm vectors $\vc{u}$ and $\vc{v}$.
Since $\supp\left(\vc{u}\pm\vc{v}\right)\cup\supp\left(\vcg{\uptheta}\right)\in\st{M}\left(\st{C}_{k}\right)$
we can use (\ref{eq:SRH}) for $\vcg{\Delta}=\vc{u}\pm\vc{v}$ to
obtain 
\begin{alignat*}{2}
\beta_{\st{C}_{k}}\norm{\vc{u}\pm\vc{v}}^{2} & \leq\left\langle \vc{u}\pm\vc{v},\nabla^{2}f\left(\vcg{\uptheta}\right)\left(\vc{u}\pm\vc{v}\right)\right\rangle  & \leq\alpha_{\st{C}_{k}}\norm{\vc{u}\pm\vc{v}}^{2}.
\end{alignat*}
These inequalities and the assumption $\norm{\vc{u}}=\norm{\vc{v}}=1$
then yield 
\begin{alignat*}{2}
\frac{\beta_{\st{C}_{k}}-\alpha_{\st{C}_{k}}}{2}+\frac{\alpha_{\st{C}_{k}}+\beta_{\st{C}_{k}}}{2}\left\langle \vc{u},\vc{v}\right\rangle  & \leq\left\langle \vc{u},\nabla^{2}f\left(\vcg{\uptheta}\right)\vc{v}\right\rangle  & \leq\frac{\alpha_{\st{C}_{k}}-\beta_{\st{C}_{k}}}{2}+\frac{\alpha_{\st{C}_{k}}+\beta_{\st{C}_{k}}}{2}\left\langle \vc{u},\vc{v}\right\rangle ,
\end{alignat*}
where we used the fact that $\nabla^{2}f\left(\vcg{\uptheta}\right)$
is symmetric since $f$ is twice continuously differentiable. Multiplying
all sides by $\eta$ and rearranging the terms then imply 
\begin{alignat}{1}
\eta\frac{\alpha_{\st{C}_{k}}-\beta_{\st{C}_{k}}}{2} & \geq\left|\left(\eta\frac{\alpha_{\st{C}_{k}}+\beta_{\st{C}_{k}}}{2}-1\right)\left\langle \vc{u},\vc{v}\right\rangle +\left\langle \vc{u},\vc{v}\right\rangle -\eta\left\langle \vc{u},\nabla^{2}f\left(\vcg{\uptheta}\right)\vc{v}\right\rangle \right|\nonumber \\
 & \geq\left|\left\langle \vc{u},\vc{v}\right\rangle -\eta\left\langle \vc{u},\nabla^{2}f\left(\vcg{\uptheta}\right)\vc{v}\right\rangle \right|-\left|\left(\eta\frac{\alpha_{\st{C}_{k}}+\beta_{\st{C}_{k}}}{2}-1\right)\left\langle \vc{u},\vc{v}\right\rangle \right|\nonumber \\
 & \geq\left|\left\langle \vc{u},\vc{v}\right\rangle -\eta\left\langle \vc{u},\nabla^{2}f\left(\vcg{\uptheta}\right)\vc{v}\right\rangle \right|-\left|\eta\frac{\alpha_{\st{C}_{k}}+\beta_{\st{C}_{k}}}{2}-1\right|,\label{eq:UnitNormIneq}
\end{alignat}
 which is equivalent to result for unit-norm $\vc{u}$ and $\vc{v}$
as desired. For the general case one can write $\vc{u}=\norm{\vc{u}}\vc{u}'$
and $\vc{v}=\norm{\vc{v}}\vc{v}'$ such that $\vc{u}'$ and $\vc{v}'$
are both unit-norm. It is straightforward to verify that using (\ref{eq:UnitNormIneq})
for $\vc{u'}$ and $\vc{v}'$ as the unit-norm vectors and multiplying
both sides of the resulting inequality by $\norm{\vc{u}}\norm{\vc{v}}$
yields the desired general case.
\end{proof}

\begin{proof}[\textbf{Proof of Theorem} \ref{thm:IterInvar}]
 Using optimality of $\vcg{\uptheta}\itr{i+1}$ and feasibility of
$\overline{\vcg{\uptheta}}$ one can deduce $\norm{\vcg{\uptheta}\itr{i+1}-\vcg{\upchi}\itr{i}}^{2}\leq\norm{\overline{\vcg{\uptheta}}-\vcg{\upchi}\itr{i}}^{2},$
with $\vcg{\upchi}\itr{i}$ as in line \ref{lin:Descent} of Algorithm
\ref{alg:GDMS}. Expanding the squared norms using the inner product
of $\st{H}$ then shows $0\leq\left\langle \vcg{\uptheta}\itr{i+1}-\overline{\vcg{\uptheta}},2\vcg{\upchi}\itr{i}\!-\!\vcg{\uptheta}\itr{i+1}-\overline{\vcg{\uptheta}}\right\rangle $
or equivalently $0\leq\left\langle \vcg{\Delta}\itr{i+1},2\vcg{\uptheta}\itr{i}\!-\!2\eta\itr{i}\nabla f\left(\overline{\vcg{\uptheta}}+\vcg{\Delta}\itr{i}\right)\!-\!\vcg{\Delta}\itr{i+1}\right\rangle ,$
where $\vcg{\Delta}\itr{i}=\vcg{\uptheta}\itr{i}\!-\!\overline{\vcg{\uptheta}}$
and $\vcg{\Delta}\itr{i+1}=\vcg{\uptheta}\itr{i+1}\!-\!\overline{\vcg{\uptheta}}$.
Adding and subtracting $2\eta\itr{i}\left\langle \vcg{\Delta}\itr{i+1},\nabla f\left(\overline{\vcg{\uptheta}}\right)\right\rangle $
and rearranging yields
\begin{alignat}{1}
\norm{\vcg{\Delta}\itr{i+1}}^{2} & \leq2\left\langle \vcg{\Delta}\itr{i+1},\vcg{\uptheta}\itr{i}\right\rangle -2\eta\itr{i}\left\langle \vcg{\Delta}\itr{i+1},\nabla f\left(\overline{\vcg{\uptheta}}+\vcg{\Delta}\itr{i}\right)-\nabla f\left(\overline{\vcg{\uptheta}}\right)\right\rangle \nonumber \\
 & -2\eta\itr{i}\left\langle \vcg{\Delta}\itr{i+1},\nabla f\left(\overline{\vcg{\uptheta}}\right)\right\rangle \label{eq:IterInvar}
\end{alignat}
 Since $f$ is twice continuously differentiable by assumption, it
follows form the mean-value theorem that $\left\langle \vcg{\Delta}\itr{i+1},\nabla f\left(\overline{\vcg{\uptheta}}+\vcg{\Delta}\itr{i}\right)-\nabla f\left(\overline{\vcg{\uptheta}}\right)\right\rangle =\left\langle \vcg{\Delta}\itr{i+1},\nabla^{2}f\left(\overline{\vcg{\uptheta}}+t\vcg{\Delta}\itr{i}\right)\vcg{\Delta}\itr{i}\right\rangle $,
for some $t\in\left(0,1\right)$. Furthermore, because $\overline{\vcg{\uptheta}}$,
$\vcg{\uptheta}\itr{i}$, $\vcg{\uptheta}\itr{i+1}$ all belong to
the model set $\st{M}\left(\st{C}_{k}\right)$ we have $\supp\left(\overline{\vcg{\uptheta}}+t\vcg{\Delta}\itr{i}\right)\in\st{M}\left(\st{C}_{k}^{2}\right)$
and thereby $\supp\left(\vcg{\Delta}\itr{i+1}\right)\cup\supp\left(\overline{\vcg{\uptheta}}+t\vcg{\Delta}\itr{i}\right)\in\st{M}\left(\st{C}_{k}^{3}\right)$.
Invoking the $\left(\alpha_{\st{C}_{k}^{3}},\beta_{\st{C}_{k}^{3}},r\right)$-SMRH
condition of the cost function and applying Lemma \ref{lem:BasicIneq}
with the sparsity model $\st{M}\left(\st{C}_{k}^{3}\right)$, $\vcg{\uptheta}=\overline{\vcg{\uptheta}}+t\vcg{\Delta}\itr{i}$,
and $\eta=\eta\itr{i}$ then yields 
\begin{alignat*}{1}
\left|\left\langle \vcg{\Delta}\itr{i+1},\vcg{\Delta}\itr{i}\right\rangle -\eta\itr{i}\left\langle \vcg{\Delta}\itr{i+1},\nabla f\left(\overline{\vcg{\uptheta}}+\vcg{\Delta}\itr{i}\right)-\nabla f\left(\overline{\vcg{\uptheta}}\right)\right\rangle \right| & \leq\gamma\itr{i}\norm{\vcg{\Delta}\itr{i+1}}\norm{\vcg{\Delta}\itr{i}}.
\end{alignat*}
 Using the Cauchy-Schwarz inequality and the fact that $\norm{\nabla_{\supp\left(\vcg{\Delta}\itr{i+1}\right)}f\left(\overline{\vcg{\uptheta}}\right)}\leq\norm{\nabla_{\overline{\st{I}}}f\left(\overline{\vcg{\uptheta}}\right)}$
by the definition of $\overline{\st{I}}$, (\ref{eq:IterInvar}) implies
that $\norm{\vcg{\Delta}\itr{i+1}}^{2}\leq2\gamma\itr{i}\norm{\vcg{\Delta}\itr{i+1}}\norm{\vcg{\Delta}\itr{i}}+2\eta\itr{i}\norm{\vcg{\Delta}\itr{i+1}}\norm{\nabla_{\overline{\st{I}}}f\left(\overline{\vcg{\uptheta}}\right)}$
. Canceling $\norm{\vcg{\Delta}\itr{i+1}}$ from both sides proves
the theorem.\end{proof}
\begin{lem}[Bounded Model Projection]
\label{lem:BoundProj}Given an arbitrary $\vc{h}_{0}\in\st{H}$,
a positive real number $r$, and a sparsity model generator $\st{C}_{k}$,
a projection $\mathrm{P}_{\st{C}_{k},r}\left[\vc{h}_{0}\right]$ can
be obtained as the projection of $\mathrm{P}_{\st{C}_{k},+\infty}\left[\vc{h}_{0}\right]$
on to the sphere of radius $r$.\end{lem}
\begin{proof}
To simplify the notation let $\widehat{\vc{h}}=\mathrm{P}_{\st{C}_{k},r}\left[\vc{h}_{0}\right]$
and $\widehat{\st{S}}=\supp\left(\widehat{\vc{h}}\right)$. For $\st{S}\subseteq\left[p\right]$
define 
\begin{alignat*}{1}
\vc{h}_{0}\left(\st{S}\right) & =\argmin_{\vc{h}}\ \norm{\vc{h}-\vc{h}_{0}}\quad\text{s.t. }\norm{\vc{h}}\leq r\text{ and }\supp\left(\vc{h}\right)\subseteq\st{S}.
\end{alignat*}
It follows from the definition of $\mathrm{P}_{\st{C}_{k},r}\left[\vc{h}_{0}\right]$
that $\widehat{\st{S}}\in\argmin_{\st{S}\in\st{C}_{k}}\ \norm{\vc{h}_{0}\left(\st{S}\right)-\vc{h}_{0}}$.
Using 
\begin{alignat*}{2}
\norm{\vc{h}_{0}\left(\st{S}\right)-\vc{h}_{0}}^{2} & =\norm{\vc{h}_{0}\left(\st{S}\right)-\res{\vc{h}_{0}}_{\st{S}}-\res{\vc{h}_{0}}_{\st{S}\cmpl}}^{2} & =\norm{\vc{h}_{0}\left(\st{S}\right)-\res{\vc{h}_{0}}_{\st{S}}}^{2}+\norm{\res{\vc{h}_{0}}_{\st{S}\cmpl}}^{2},
\end{alignat*}
 we deduce that $\vc{h}_{0}\left(\st{S}\right)$ is the projection
of $\res{\vc{h}_{0}}_{\st{S}}$ onto the sphere of radius $r$. Therefore,
we can write $\vc{h}_{0}\left(\st{S}\right)=\min\left\{ 1,r/\norm{\res{\vc{h}_{0}}_{\st{S}}}\right\} \res{\vc{h}_{0}}_{\st{S}}$
and from that 
\begin{alignat*}{1}
\widehat{\st{S}} & \in\argmin_{\st{S}\in\st{C}_{k}}\ \norm{\min\left\{ 1,r/\norm{\res{\vc{h}_{0}}_{\st{S}}}\right\} \res{\vc{h}_{0}}_{\st{S}}-\vc{h}_{0}}^{2}\\
 & =\argmin_{\st{S}\in\st{C}_{k}}\ \norm{\min\left\{ 0,r/\norm{\res{\vc{h}_{0}}_{\st{S}}}-1\right\} \res{\vc{h}_{0}}_{\st{S}}}^{2}+\norm{\res{\vc{h}_{0}}_{\st{S}\cmpl}}^{2}\\
 & =\argmin_{\st{S}\in\st{C}_{k}}\ \left(\left(1-r/\norm{\res{\vc{h}_{0}}_{\st{S}}}\right)_{+}^{2}-1\right)\norm{\res{\vc{h}_{0}}_{\st{S}}}^{2}\\
 & =\argmax_{\st{S}\in\st{C}_{k}}\ q\left(\st{S}\right):=\norm{\res{\vc{h}_{0}}_{\st{S}}}^{2}-\left(\norm{\res{\vc{h}_{0}}_{\st{S}}}-r\right)_{+}^{2}.
\end{alignat*}
Furthermore, let 
\begin{alignat}{1}
\st{S}_{0} & =\supp\left(\mathrm{\mathrm{P}_{\st{C}_{k},+\infty}\left[\vc{h}_{0}\right]}\right)=\argmax_{\st{S}\in\st{C}_{k}}\ \norm{\res{\vc{h}_{0}}_{\st{S}}}.\label{eq:SupportOpt}
\end{alignat}
If $\norm{\res{\vc{h}_{0}}_{\st{S}_{0}}}\leq r$ then $q\left(\st{S}\right)=\norm{\res{\vc{h}_{0}}_{\st{S}}}\leq q\left(\st{S}_{0}\right)$
for any $\st{S}\in\st{C}_{k}$ and thereby $\widehat{\st{S}}=\st{S}_{0}$.
Thus, we focus on cases that $\norm{\res{\vc{h}_{0}}_{\st{S}_{0}}}>r$
which implies $q\left(\st{S}_{0}\right)=2\norm{\res{\vc{h}_{0}}_{\st{S}_{0}}}r-r^{2}$.
For any $\st{S}\in\st{C}_{k}$ if $\norm{\res{\vc{h}_{0}}_{\st{S}}}\leq r$
we have $q\left(\st{S}\right)=\norm{\res{\vc{h}_{0}}_{\st{S}}}^{2}\leq r^{2}<2\norm{\res{\vc{h}_{0}}_{\st{S}_{0}}}r-r^{2}=q\left(\st{S}_{0}\right)$,
and if $\norm{\res{\vc{h}_{0}}_{\st{S}}}>r$ we have $q\left(\st{S}\right)=2\norm{\res{\vc{h}_{0}}_{\st{S}}}r-r^{2}\leq2\norm{\res{\vc{h}_{0}}_{\st{S}_{0}}}r-r^{2}=q\left(\st{S}_{0}\right)$
where (\ref{eq:SupportOpt}) is applied. Therefore, we have shown
that $\widehat{\st{S}}=\st{S}_{0}$. It is then straightforward to
show the desired result that projecting $\mathrm{P}_{\st{C}_{k},+\infty}\left[\vc{h}_{0}\right]$
onto the centered sphere of radius $r$ yields $\mathrm{P}_{\st{C}_{k},r}\left[\vc{h}_{0}\right]$.
\end{proof}
{\small{}\bibliographystyle{IEEEtran}
\bibliography{references}

\begin{thebibliography}{10}
\providecommand{\url}[1]{#1}
\csname url@samestyle\endcsname
\providecommand{\newblock}{\relax}
\providecommand{\bibinfo}[2]{#2}
\providecommand{\BIBentrySTDinterwordspacing}{\spaceskip=0pt\relax}
\providecommand{\BIBentryALTinterwordstretchfactor}{4}
\providecommand{\BIBentryALTinterwordspacing}{\spaceskip=\fontdimen2\font plus
\BIBentryALTinterwordstretchfactor\fontdimen3\font minus
  \fontdimen4\font\relax}
\providecommand{\BIBforeignlanguage}[2]{{%
\expandafter\ifx\csname l@#1\endcsname\relax
\typeout{** WARNING: IEEEtran.bst: No hyphenation pattern has been}%
\typeout{** loaded for the language `#1'. Using the pattern for}%
\typeout{** the default language instead.}%
\else
\language=\csname l@#1\endcsname
\fi
#2}}
\providecommand{\BIBdecl}{\relax}
\BIBdecl

\bibitem{bach_consistency_2008}
F.~R. Bach, ``Consistency of the group lasso and multiple kernel learning,''
  \emph{Journal of Machine Learning Research}, vol.~9, pp. 1179--1225, Jun.
  2008.

\bibitem{roth_group-lasso_2008}
V.~Roth and B.~Fischer, ``The {Group-Lasso} for generalized linear models:
  uniqueness of solutions and efficient algorithms,'' in \emph{Proceedings of
  the 25th International Conference on Machine learning}, ser. {ICML} '08, New
  York, {NY}, {USA}, 2008, pp. 848--855.

\bibitem{jacob_group_2009}
L.~Jacob, G.~Obozinski, and J.~Vert, ``Group lasso with overlap and graph
  lasso,'' in \emph{Proceedings of the 26th Annual International Conference on
  Machine Learning}, ser. {ICML} '09, New York, {NY}, {USA}, 2009, pp.
  433--440.

\bibitem{baraniuk_model-based_2010}
R.~G. Baraniuk, V.~Cevher, M.~F. Duarte, and C.~Hegde, ``Model-based
  compressive sensing,'' \emph{{IEEE} Transactions on Information Theory},
  vol.~56, no.~4, pp. 1982--2001, 2010.

\bibitem{bach_structured_2010}
F.~Bach, ``Structured sparsity-inducing norms through submodular functions,''
  in \emph{Advances in Neural Information Processing Systems 23}, Vancouver,
  {BC}, Canada, Dec. 2010, pp. 118--126.

\bibitem{jenatton_structured_2011}
R.~Jenatton, J.~Audibert, and F.~Bach, ``Structured variable selection with
  {Sparsity-Inducing} norms,'' \emph{Journal of Machine Learning Research},
  vol.~12, pp. 2777--2824, Oct. 2011.

\bibitem{kyrillidis_combinatorial_2012}
A.~Kyrillidis and V.~Cevher, ``Combinatorial selection and least absolute
  shrinkage via the {CLASH} algorithm,'' in \emph{Information Theory
  Proceedings (ISIT), IEEE International Symposium on}, Jul. 2012, pp.
  2216--2220.

\bibitem{chandrasekaran_convex_2010}
V.~Chandrasekaran, B.~Recht, P.~A. Parrilo, and A.~S. Willsky, ``The convex
  geometry of linear inverse problems,'' \emph{Foundation of Computational
  Mathematics}, vol.~12, no.~6, pp. 805--849, 2012.

\bibitem{bach_structured_2011}
F.~Bach, R.~Jenatton, J.~Mairal, and G.~Obozinski, ``Structured sparsity
  through convex optimization,'' \emph{Statistical Science}, vol.~27, no.~4,
  pp. 450--468, Nov. 2012.

\bibitem{duarte_structured_2011}
M.~Duarte and Y.~Eldar, ``Structured compressed sensing: From theory to
  applications,'' \emph{Signal Processing, {IEEE} Transactions on}, vol.~59,
  no.~9, pp. 4053--4085, Sep. 2011.

\bibitem{tewari_greedy_2011}
A.~Tewari, P.~K. Ravikumar, and I.~S. Dhillon, ``Greedy algorithms for
  structurally constrained high dimensional problems,'' in \emph{Advances in
  Neural Information Processing Systems 24}, J.~Shawe-Taylor, R.~Zemel,
  P.~Bartlett, F.~Pereira, and K.~Weinberger, Eds., 2011, pp. 882--890.

\bibitem{blumensath_compressed_2010}
\BIBentryALTinterwordspacing
T.~Blumensath, ``Compressed sensing with nonlinear observations,'' 2010,
  preprint. [Online]. Available:
  \url{http://eprints.soton.ac.uk/164753/1/B_Nonlinear.pdf}
\BIBentrySTDinterwordspacing

\bibitem{lozano_group_2011}
A.~Lozano, G.~Swirszcz, and N.~Abe, ``Group orthogonal matching pursuit for
  logistic regression,'' in \emph{the Fourteenth International Conference on
  Artificial Intelligence and Statistics}, G.~Gordon, D.~Dunson, and M.~Dudik,
  Eds., vol.~15.\hskip 1em plus 0.5em minus 0.4em\relax Ft. Lauderdale, {FL},
  {USA}: {JMLR} {W\&CP}, 2011, pp. 452--460.

\bibitem{pati_orthogonal_1993}
Y.~C. Pati, R.~Rezaiifar, and P.~S. Krishnaprasad, ``Orthogonal matching
  pursuit: Recursive function approximation with applications to wavelet
  decomposition,'' in \emph{Conference Record of the 27th Asilomar Conf. on
  Signals, Systems and Computers}, vol.~1, Pacific Grove, {CA}, Nov. 1993, pp.
  40--44.

\bibitem{beck_sparsity_2012}
A.~Beck and Y.~C. Eldar, ``Sparsity constrained nonlinear optimization:
  Optimality conditions and algorithms,'' \emph{{SIAM} Journal on
  Optimization}, vol.~23, no.~3, pp. 1480--1509, 2013.

\bibitem{blumensath_compressed_2013}
T.~Blumensath, ``Compressed sensing with nonlinear observations and related
  nonlinear optimization problems,'' \emph{{IEEE} Transactions on Information
  Theory}, vol.~59, no.~6, pp. 3466--3474, 2013.

\bibitem{bunea_honest_2008}
F.~Bunea, ``Honest variable selection in linear and logistic regression models
  via $\ell_1$ and $\ell_1+\ell_2$ penalization,'' \emph{Electronic Journal of
  Statistics}, vol.~2, pp. 1153--1194, 2008.

\bibitem{negahban_unified_2009}
S.~Negahban, P.~Ravikumar, M.~Wainwright, and B.~Yu, ``A unified framework for
  high-dimensional analysis of {$M$-estimators} with decomposable
  regularizers,'' in \emph{Advances in Neural Information Processing Systems
  22}, Vancouver, {BC}, Canada, Dec. 2009, pp. 1348--1356, long version
  available at \href{http://arxiv.org/abs/1010.2731v1}{\tt arXiv:1010.2731v1
  [math.ST]}.

\bibitem{agarwal_fast_2010}
A.~Agarwal, S.~Negahban, and M.~Wainwright, ``Fast global convergence rates of
  gradient methods for high-dimensional statistical recovery,'' in
  \emph{Advances in Neural Information Processing Systems 23}, Vancouver, {BC},
  Canada, 2010, pp. 37--45, long version available at
  \href{http://arxiv.org/abs/1104.4824v1}{\tt arXiv:1104.4824v1 [stat.ML]}.

\bibitem{kakade_learning_2010}
S.~M. Kakade, O.~Shamir, K.~Sridharan, and A.~Tewari, ``Learning exponential
  families in {High-Dimensions:} strong convexity and sparsity,'' in
  \emph{Proceedings of the 13th International Conference on Artificial
  Intelligence and Statistics}, ser. {JMLR} Workshop and Conference
  Proceedings, vol.~9, Sardinia, Italy, 2010, pp. 381--388.

\bibitem{hegde_approximation-tolerant_2014}
C.~Hegde, P.~Indyk, and L.~Schmidt, ``Approximation-tolerant model-based
  compressive sensing,'' in \emph{Proceedings of the Twenty-Fifth Annual
  {ACM-SIAM} Symposium on Discrete Algorithms}, 2014, ch. 113, pp. 1544--1561.

\bibitem{B_TIT_11}
T.~Blumensath, ``Sampling and reconstructing signals from a union of linear
  subspaces,'' \emph{Information Theory, IEEE Transactions on}, vol.~57, no.~7,
  pp. 4660--4671, 2011.

\bibitem{DNW_TIT13}
M.~Davenport, D.~Needell, and M.~Wakin, ``Signal space {CoSaMP} for sparse
  recovery with redundant dictionaries,'' \emph{{IEEE} Transactions on
  Information Theory}, vol.~59, no.~10, pp. 6820--6829, 2013.

\bibitem{GNEGD_LAA14}
R.~Giryes, S.~Nam, M.~Elad, R.~Gribonval, and M.~Davies, ``Greedy-like
  algorithms for the cosparse analysis model,'' \emph{Linear Algebra and its
  Applications, Special Issue on Sparse Approximate Solution of Linear
  Systems}, vol. 441, pp. 22--60, 2014.

\bibitem{bahmani_GraSP_2012}
S.~Bahmani, B.~Raj, and P.~Boufounos, ``Greedy sparsity-constrained
  optimization,'' \emph{Journal of Machine Learning Research}, vol.~14, no.~3,
  pp. 807--841, 2013.

\bibitem{van_de_geer_empirical_2000}
S.~A. van~de Geer, \emph{Empirical processes in M-estimation}.\hskip 1em plus
  0.5em minus 0.4em\relax Cambridge, {UK}: Cambridge University Press, 2000.

\bibitem{vapnik_statistical_1998}
V.~Vapnik, \emph{Statistical learning theory}.\hskip 1em plus 0.5em minus
  0.4em\relax New York, {NY}: Wiley, 1998.

\bibitem{dobson_introduction_2008}
A.~J. Dobson and A.~Barnett, \emph{An Introduction to Generalized Linear
  Models}, 3rd~ed.\hskip 1em plus 0.5em minus 0.4em\relax Boca Reaton, {FL}:
  Chapman and {Hall/CRC}, May 2008.

\bibitem{tropp_user-friendly_2011}
J.~A. Tropp, ``User-friendly tail bounds for sums of random matrices,''
  \emph{Foundations of Computational Mathematics}, vol.~12, no.~4, pp.
  389--434, 2012.

\bibitem{hsu_tail_2011}
D.~Hsu, S.~Kakade, and T.~Zhang, ``Tail inequalities for sums of random
  matrices that depend on the intrinsic dimension,'' \emph{Electron. Commun.
  Probab.}, vol.~17, no.~14, pp. 1--13, 2012.

\end{thebibliography}
}
\end{document}